\documentclass{article}

\PassOptionsToPackage{numbers, compress}{natbib}
\usepackage[main, final]{neurips_2026}
\usepackage[utf8]{inputenc} % allow utf-8 input
\usepackage[T1]{fontenc}    % use 8-bit T1 fonts
\usepackage{hyperref}       % hyperlinks
\usepackage{url}            % simple URL typesetting
\usepackage{booktabs}       % professional-quality tables
\usepackage{amsfonts}       % blackboard math symbols
\usepackage{nicefrac}       % compact symbols for 1/2, etc.
\usepackage{microtype}      % microtypography
\usepackage{xcolor}         % colors

\usepackage{amsmath}
\usepackage{amssymb}
\usepackage{mathtools}
\usepackage{amsthm}
\usepackage{eccvabbrv}
\usepackage{caption}
\usepackage{enumitem}
\usepackage{wrapfig}
\usepackage{graphicx}
\usepackage{multirow}
\usepackage[normalem]{ulem}
\useunder{\uline}{\ul}{}

\theoremstyle{plain}
\newtheorem{theorem}{Theorem}[section]
\newtheorem{proposition}[theorem]{Proposition}

\theoremstyle{definition}

\theoremstyle{remark}

\title{Metric Depth Estimation from Arbitrarily Degraded Low-Resolution Depth Prompts}

\author{
	Kun Wang\textsuperscript{\rm 1},\ Yun Zhu\textsuperscript{\rm 2},\ Pan Zhou\textsuperscript{\rm 3} and Na Zhao\textsuperscript{\rm 1}\thanks{Corresponding author: \texttt{na\_zhao@sutd.edu.sg}.}\\
	\textsuperscript{\rm 1}Singapore University of Technology and Design\\
	\textsuperscript{\rm 2}Nanjing University of Science and Technology\ \ \ \ 
	\textsuperscript{\rm 3}Singapore Management University
}
\begin{document}

\maketitle

\begin{figure}[h]
	\vskip -0.1in
	\centering
	\includegraphics[width=\linewidth]{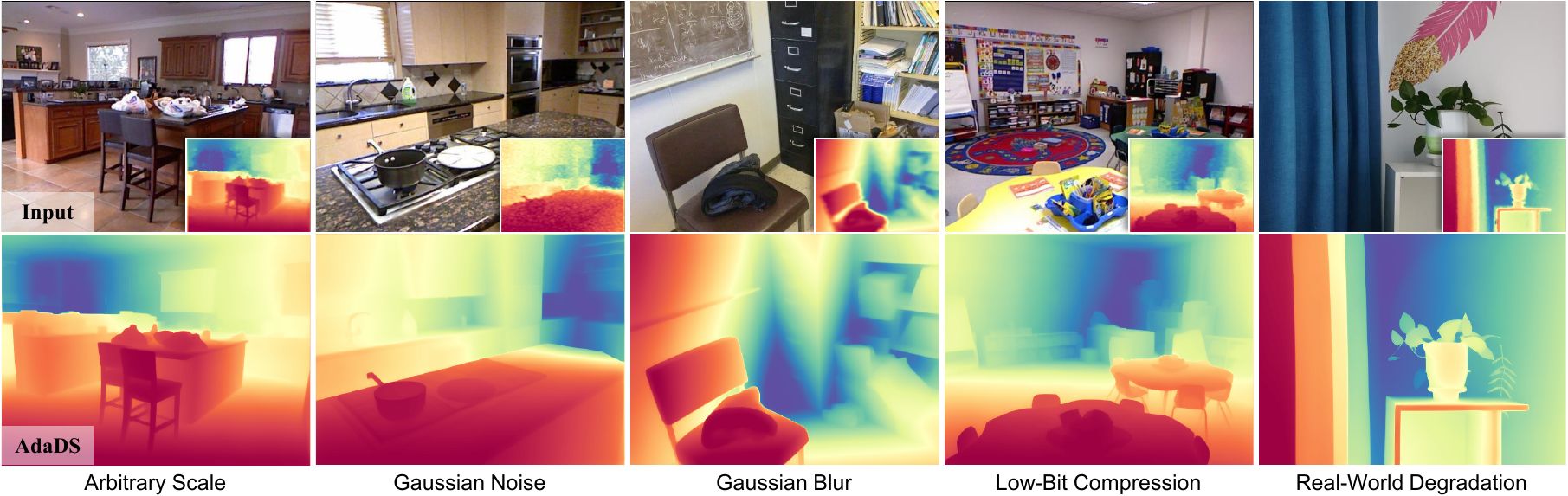}
	\caption{
		Given an RGB image and an arbitrarily degraded low-resolution depth prompt, \textbf{AdaDS} adaptively places the intermediate noisy sample in a high-density region of a pre-trained depth diffusion prior, enabling zero-shot estimation of fine-grained high-resolution metric depth.
	}
	\label{fig:teaser}
\end{figure}

\begin{abstract}
%	Depth Super-Resolution (DSR) aims to recover an accurate high-resolution depth map from coarse, low-resolution measurement under the guidance of color image. Due to various degradation of low-resolution measurement, such as blur, noise and inaccuracy,  
%	We propose AdaDS, a generalizable framework for depth super-resolution that robustly recovers high-resolution depth maps from arbitrarily degraded low-resolution inputs. Unlike conventional approaches that directly regress depth values and often exhibit artifacts under severe or unknown degradation, AdaDS capitalizes on the contraction property of Gaussian smoothing: as noise accumulates in the forward process, distributional discrepancies between degraded inputs and their pristine high-quality counterparts diminish, ultimately converging to isotropic Gaussian prior. Leveraging this, AdaDS adaptively selects a starting timestep in the reverse diffusion trajectory based on estimated refinement uncertainty, and subsequently injects tailored noise to position the intermediate sample within the high-probability region of the target posterior distribution. This strategy ensures inherent robustness, enabling generative prior of a pre-trained diffusion model to dominate recovery even when upstream estimations are imperfect. Extensive experiments on real-world and synthetic benchmarks demonstrate AdaDS's superior zero-shot generalization and resilience to diverse degradation patterns compared to state-of-the-art methods.
	
We propose AdaDS, an \textbf{Ada}ptive \textbf{D}iffusion \textbf{S}ampling framework for prompted metric depth estimation, which estimates high-resolution metric depth from an image and an arbitrarily degraded low-resolution depth prompt. This setting is commonly studied as depth super-resolution, where existing methods typically regress depth directly without an explicit prior over high-quality depth, and thus often produce artifacts under severe or unseen degradations. In contrast, AdaDS exploits the contraction property of Gaussian smoothing: as noise accumulates in the forward diffusion process, the noisy distributions of degraded prompts and their high-quality counterparts converge toward an isotropic Gaussian prior, while the prompt content is gradually erased. Accordingly, AdaDS refines the prompt and estimates its uncertainty to adaptively select a starting timestep, and injects learned noise to place the intermediate sample in a high-density region of the depth distribution learned by a pre-trained diffusion model. Denoising from this sample yields high-quality depth with fine-grained details and corrects residual errors, even under imperfect prompt refinement. Experiments on real-world and synthetic benchmarks demonstrate its superior zero-shot generalization and robustness to diverse degradations. Code is available at \url{https://github.com/w2kun/AdaDS}.

\end{abstract}

\section{Introduction}

Monocular depth estimation is fundamental to 3D perception \cite{dcdepth, yuan2026vggt, ccf, Zhu_2026_CVPR}, yet it remains limited by inherent scale and geometric ambiguities. In practical settings, low-cost range sensors, such as Time-of-Flight (ToF) cameras and structured-light systems, provide coarse but valuable metric depth cues. We investigate \textit{prompted metric depth estimation}, which estimates high-resolution metric depth from an image and an arbitrarily degraded low-resolution depth prompt. This formulation subsumes conventional depth super-resolution (DSR), which typically treats the low-resolution depth as reliable and focuses on recovering high-frequency details, whereas we regard it as a geometric prompt that may contain substantial errors. By combining RGB semantics with metric cues from depth prompts, this task reduces the cost of acquiring high-quality depth data and benefits downstream applications such as 3D scene reconstruction \cite{chung2024depth, jiang2025gaussianblock, jia2025h3r, xiu2025geometric, shen2025gamba} and embodied reasoning \cite{yuan2025scene, wang2026affordbot}.

Numerous methods have achieved remarkable results under controlled DSR settings, such as fixed upsampling factors, specific scene types, and limited degradation patterns \cite{suft, sgnet}. However, real-world depth prompts acquired from hardware-constrained sensors often exhibit complex and unpredictable degradations, making robustness a critical limitation for practical deployment. Most existing approaches directly map degraded prompts to depth \cite{promptda, priorda}, so prompt errors propagate into their predictions and severely degrade their real-world performance.

\begin{wrapfigure}{r}{0.5\linewidth}
	\vspace{-0.15in}
	\centering
	\includegraphics[width=\linewidth]{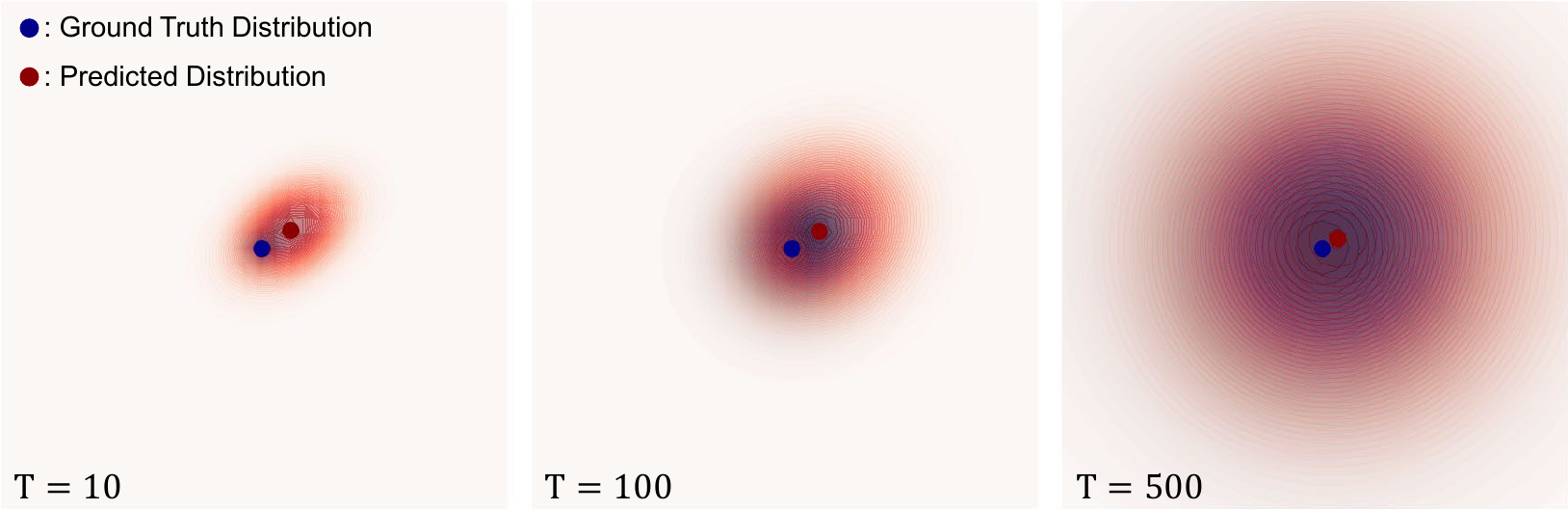}
	\caption{
		The distribution of the degraded depth prompt (with predicted mean and variance from coarse refinement) progressively aligns with the ground-truth distribution due to Gaussian smoothing contraction. Dots indicate the means.
	}
	\label{fig:fig2}
	\vspace{-0.1in}
\end{wrapfigure}

To address this challenge, we propose AdaDS, a robust framework for prompted metric depth estimation under arbitrary depth degradation. Rather than strictly following the prompt, AdaDS projects its estimate onto a high-density region of the high-quality depth distribution learned by a pre-trained diffusion model \cite{ddpm, marigold}, allowing the generative prior to correct prompt errors. Our key insight is that the forward diffusion process progressively reduces the distributional discrepancy between degraded depth prompts and their high-quality counterparts through Gaussian smoothing contraction, eventually converging both distributions to a shared isotropic Gaussian prior (see the 2D Gaussian analogy in Fig. \ref{fig:fig2}). Since increasing noise also erases the prompt content, there exists an optimal timestep that balances prompt preservation and alignment with the high-quality depth distribution.

Based on this insight, AdaDS operates in two stages. The first stage performs coarse prompt refinement to produce an initial depth estimate and a variance map that quantifies refinement uncertainty. This uncertainty is then used to adaptively select a timestep that preserves reliable prompt information while improving alignment with the high-quality depth prior. In the second stage, a dedicated noise sampling module predicts timestep-conditioned noise that, when added to the coarse estimate, places the intermediate latent in a high-density region of the high-quality depth distribution. This sample is then denoised by a pre-trained depth diffusion model to produce the final high-resolution metric depth map. By design, imperfect upstream prompt refinement can be tolerated, as adequate noising projects the sample onto the learned generative prior, yielding robustness to diverse depth degradations.

The main contributions of this work are:
\begin{itemize}[leftmargin=20pt,itemsep=5pt,topsep=0pt,parsep=0pt]
	\item We introduce AdaDS, a generalizable framework for prompted metric depth estimation that effectively addresses robustness challenges under arbitrary upsampling factors, diverse degradation patterns, and unseen real-world scenes.
	\item We propose a novel adaptive sampling strategy that exploits the contraction property of the forward diffusion process to place an intermediate noisy sample in a high-density region of the high-quality depth distribution, thereby leveraging the powerful generative prior of a relative depth diffusion model for robust metric depth estimation.
	\item We conduct extensive experiments on both real-world and synthetic benchmarks, demonstrating AdaDS's superior zero-shot generalization and robustness compared to existing approaches.
\end{itemize}

\section{Related Work}

\paragraph{Depth Super-Resolution.}
Metric depth estimation from low-resolution depth prompts has commonly been studied as depth super-resolution (DSR). Early DSR works restore high-frequency details lost in synthetically downsampled depth maps \cite{guo2018hierarchical, dctnet, yuan2023recurrent}. The subsequent availability of real-world datasets \cite{rgbdd, tofdsr} has motivated specialized designs for more complex degradations, such as structure flow guidance \cite{sfg}, degradation-driven depth bins \cite{degbins}, degradation-aware representation learning \cite{dornet}, and semantic-driven contrastive learning \cite{sdcl}, along with other advanced architectures \cite{sgnet, c2pd}. Despite impressive in-domain performance, these methods are tailored to specific degradation assumptions and generalize poorly to unseen degradations and scenes.

\paragraph{Prompted Metric Depth Estimation.}
Recent works leverage depth foundation models to improve generalization across scenes and prompt patterns. PromptDA \cite{promptda} integrates low-cost LiDAR prompts into the decoder of Depth Anything \cite{dav2} at multiple scales. PriorDA \cite{priorda} pre-fills diverse metric priors through pixel-level alignment with monocular predictions and refines them with a conditioned depth estimator. OMNI-DC \cite{omni_dc} learns robust depth completion from synthetic data with diverse sparsity patterns. Despite strong zero-shot performance, these methods produce depth through deterministic feed-forward regression, so their robustness to prompt errors largely depends on the degradations simulated during training. Marigold-DC \cite{marigold_dc} instead injects sparse depth into a pre-trained depth diffusion model as test-time guidance via iterative optimization, which also enforces consistency with erroneous measurements. In contrast, AdaDS determines how much to trust the prompt via estimated uncertainty and projects its estimate onto a high-density region of the high-quality depth distribution, enabling robust metric depth estimation from arbitrarily degraded low-resolution depth prompts.

\paragraph{Depth Diffusion Models.}
Diffusion models have recently emerged as a powerful approach for high-fidelity depth estimation. Marigold \cite{marigold} pioneered this direction by fine-tuning a Stable Diffusion \cite{stable_diffusion} model on limited synthetic data, enabling strong zero-shot performance with fine-grained outputs. Subsequently, GeoWizard \cite{geowizard} jointly estimates depth and surface normals for enhanced geometric consistency; Lotus \cite{lotus} mitigates undesirable variance in the generative process by directly predicting deterministic annotations; DepthFM \cite{depthfm} accelerates sampling through flow matching; and Jasmine \cite{jasmine} extends diffusion-based depth estimation to self-supervised learning. Furthermore, BetterDepth \cite{betterdepth} and SharpDepth \cite{sharpdepth} combine the generalization of discriminative pre-training with the detail recovery capability of diffusion models, balancing accuracy and visual sharpness. Unlike these methods that generate depth from pure noise, AdaDS starts denoising from an adaptively selected intermediate timestep anchored to the refined prompt.
\section{Depth Diffusion Model Preliminary}

Our framework employs the pre-trained Marigold-LCM \cite{marigold, marigold_ext} as the denoiser, benefiting from its strong modeling of high-quality depth distributions and efficient one-step inference capability. It is distilled from a Denoising Diffusion Probabilistic Model (DDPM) \cite{ddpm} based teacher model via latent consistency distillation \cite{consistency_models, lcm}, and thus shares the same forward process as the teacher.

Specifically, given an RGB image $\mathbf{x}$ and its corresponding depth map $\mathbf{d}$, both are first encoded into the latent space $\mathbf{z}_x=\mathcal{E}_l(\mathbf{x})$ and $\mathbf{z}_d=\mathcal{E}_l(\mathbf{d})$ via a frozen VAE \cite{vae} encoder $\mathcal{E}_l(\cdot)$, where the depth map is affinely normalized to $[-1,1]$ before encoding. In the forward diffusion process, the depth latent $\mathbf{z}_d$ is progressively corrupted by Gaussian noise according to the conditional distribution:
\begin{equation}
	q(\mathbf{z}_t\mid\mathbf{z}_d) = \mathcal{N}\bigl(\sqrt{\bar{\alpha}_t} \,\mathbf{z}_d, \; (1 - \bar{\alpha}_t)\mathbf{I}\bigr),
	\label{eq:gt_forward}
\end{equation}
where $\bar{\alpha}_t = \prod_{s=1}^t (1 - \beta_s) \in (0,1]$ is the cumulative product of the noise schedule $\beta_s$, and as $t \to T$, $\mathbf{z}_t$ approximately follows a standard Gaussian distribution.
Consistency models \cite{consistency_models} learn a consistency function $f_\theta$ that maps any noisy sample on a probability flow ODE trajectory directly to its origin. Accordingly, Marigold-LCM estimates the clean depth latent from a noisy latent at an arbitrary timestep $t$ in a single step:
\begin{equation}
	\hat{\mathbf{z}} = f_\theta\bigl(\mathbf{z}_x \Vert \mathbf{z}_t,\, t\bigr),
	\label{eq:lcm_inference}
\end{equation}
where $\Vert$ denotes channel-wise concatenation. Although standard inference starts from pure noise $\mathbf{z}_T=\epsilon\sim\mathcal{N}(\mathbf{0},\mathbf{I})$, this property allows denoising to be initiated from an intermediate timestep. The final depth map is then recovered as $\hat{\mathbf{d}} = \mathcal{D}_l(\hat{\mathbf{z}})$ via the pre-trained VAE decoder $\mathcal{D}_l(\cdot)$.

\begin{figure}
	\centering
	\includegraphics[width=\linewidth]{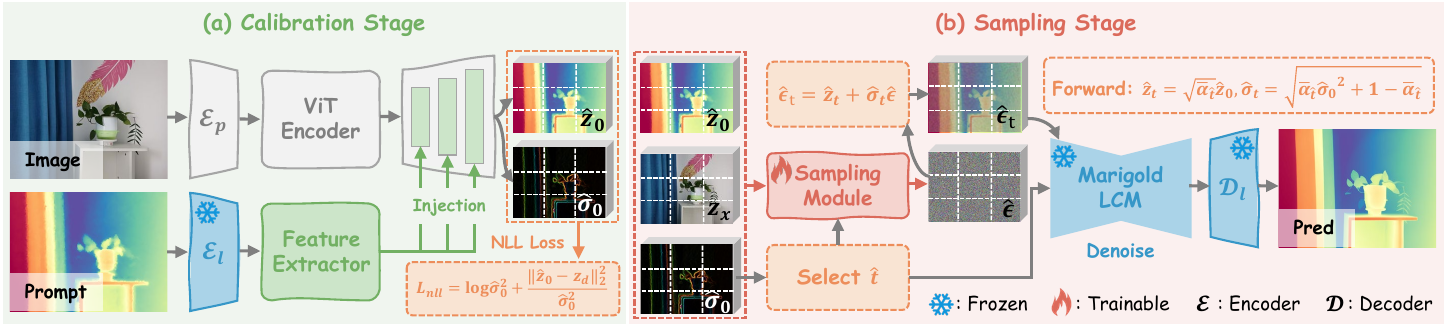}
	\caption{		
		\textbf{Overall framework of AdaDS.} (a) The first stage produces a refined depth latent along with an estimated uncertainty map reflecting the refinement reliability. (b) The second stage blends the coarse refinement with estimated noise, placing it in a high-density region of the high-quality depth distribution. The resulting intermediate noisy latent is denoised by a pre-trained relative depth diffusion model to recover the final depth map. $\mathcal{E}_p$, $\mathcal{E}_l$ and $\mathcal{D}_l$ denote the patch embedding, latent encoder and latent decoder, respectively.
	}
	\label{fig:framework}
	\vskip -0.05in
\end{figure}

\section{AdaDS Framework}
\label{sec:framework}
AdaDS is motivated by the contraction property of the forward diffusion process. Let $\mathbf{d}_{in}$ denote the degraded low-resolution (LR) depth prompt and $\mathbf{z}_{in} = \mathcal{E}_l(\mathbf{d}_{in})$ its latent representation. As shown in Eq. \eqref{eq:gt_forward}, the discrepancy between the forward distributions $q(\mathbf{z}_t \mid \mathbf{z}_d)$ and $q(\mathbf{z}_t \mid \mathbf{z}_{in})$ diminishes as $\bar{\alpha}_t$ approaches 0, while the prompt content is increasingly erased. This implies an optimal timestep that balances prompt fidelity and alignment with the high-quality depth distribution. Accordingly, AdaDS perturbs a refined depth latent with learned noise at the selected timestep, placing the noisy sample in a high-density region of the high-quality depth distribution learned by a pre-trained diffusion model, from which the final depth map is recovered by denoising. As illustrated in Fig. \ref{fig:framework}, AdaDS is implemented as a two-stage pipeline.

\paragraph{Calibration Stage.}
The calibration stage is illustrated in Fig. \ref{fig:framework} (a). It refines the LR depth prompt with two objectives: (1) reducing the gap between the LR prompt and its high-quality counterpart, thereby avoiding excessively large timesteps that would erase important prompt content; and (2) estimating the predictive uncertainty of the refinement, which serves as the basis for adaptive timestep selection in the second stage. 

To achieve these goals, we design a two-branch module $\mathcal{R}(\cdot)$ that takes the image $\mathbf{x}$ and interpolated $\mathbf{d}_{in}$ as inputs, producing a refined depth latent and its associated uncertainty: $(\hat{\mathbf{z}}_0, \hat{\sigma}_0) = \mathcal{R}(\mathbf{x}, \mathbf{d}_{in})$,
%\begin{equation}
%	(\hat{\mathbf{z}}_0, \hat{\sigma}_0) = \mathcal{R}(\mathbf{x}, \mathbf{d}_{in}),
%\end{equation}
where $\hat{\mathbf{z}}_0$ is the refined depth latent and $\hat{\sigma}_0$ denotes the predicted standard deviation at each spatial location. 
The image branch fine-tunes a variant of Depth Anything V2 \cite{dav2} to exploit its strong spatial perception capability. It consists of a patch embedding layer $\mathcal{E}_p(\cdot)$ that downsamples the input image $\mathbf{x}$ by a factor of 16, followed by a sequence of ViT \cite{vit} blocks for global feature extraction, and a DPT \cite{dpt} head for progressive feature upsampling. The depth branch first encodes $\mathbf{d}_{in}$ into the latent space $\mathbf{z}_{in}=\mathcal{E}_l(\mathbf{d}_{in})$ at $1/8$ scale. A UNet-like architecture \cite{unet} then extracts multi-scale features at $1/32$, $1/16$, and $1/8$ scales, which are injected into the DPT head of the image branch via element-wise addition to enhance depth-aware guidance. Training uses two complementary losses. First, we adopt the Negative Log-Likelihood (NLL) loss \cite{nll_loss} for well-calibrated uncertainty estimation:
\begin{equation}
	\mathcal{L}_{\text{nll}} = \log \hat{\sigma}_0^2 + \frac{\|\hat{\mathbf{z}}_0 - \mathbf{z}_d\|_2^2}{\hat{\sigma}_0^2},
\end{equation}
where $\mathbf{z}_d$ is the ground-truth depth latent. Second, to compensate for potential reconstruction errors introduced by the VAE encoder/decoder pair, we add an $\ell_1$ supervision on the decoded depth:
\begin{equation}
	\mathcal{L}_d = \|\mathcal{D}_l(\hat{\mathbf{z}}_0) - \mathbf{d}\|_1.
\end{equation}

\paragraph{Sampling Stage.}

As shown in Fig. \ref{fig:framework} (b), the second stage builds upon the outputs of the calibration stage to define an approximate forward diffusion process centered on the refined latent:
\begin{equation}
%	p(\hat{\mathbf{z}}_t|\hat{\mathbf{z}}_0, \hat{\sigma}_0,\mathbf{z}_x) = \mathcal{N}\bigl(\sqrt{\bar{\alpha}_t} \,\hat{\mathbf{z}}_0, \; (\bar{\alpha}_t\hat{\sigma}_0 + 1 - \bar{\alpha}_t)\mathbf{I}\bigr).
	\hat{q}(\mathbf{z}_t \mid \hat{\mathbf{z}}_0, \hat{\sigma}_0) = \mathcal{N}\bigl( \sqrt{\bar{\alpha}_t} \,\hat{\mathbf{z}}_0, \; \bar{\alpha}_t \hat{\sigma}_0^2 + 1 - \bar{\alpha}_t \bigr),
	\label{eq:pred_forward}
\end{equation}
where the variance is defined element-wise.
\begin{figure}
	\centering
	\includegraphics[width=\textwidth]{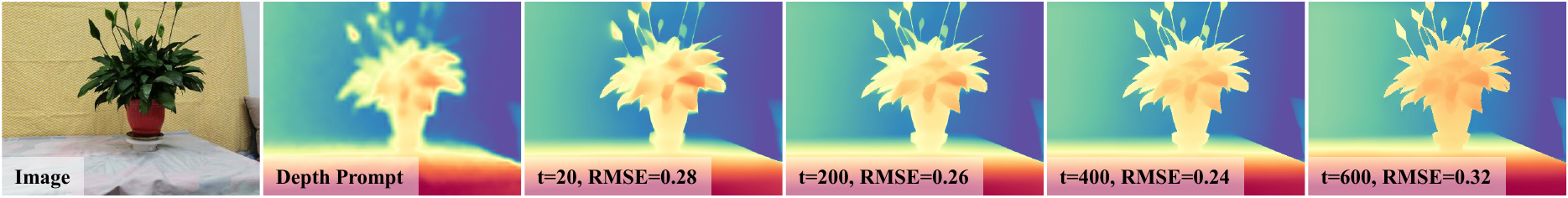}
	\caption{
		As $t$ increases, the reconstructed depth maps exhibit sharper structural details but deviate further from the input depth prompt, so that the error first decreases and then increases. All depth maps share a unified color palette, with RMSE reported in centimeters.
		%		AdaDS outputs visualizations of depth maps at different timesteps $t$. As $t$ increases, the quality of the generated depth maps improves, but the differences from the input depth maps also increase. All depth maps are visualized using a unified color palette.
	}
	\label{fig:t_demo}
	\vskip -0.1in
\end{figure}
Before detailing the sampling procedure, we first introduce and prove a \textit{proposition} that underpins our adaptive timestep selection strategy. Let $\mathcal{W}(\cdot,\cdot)$ denote the 2-Wasserstein distance, $\hat{q}_t$ the approximate forward distribution in Eq. \eqref{eq:pred_forward}, and $q_t$ the ground-truth forward distribution in Eq. \eqref{eq:gt_forward}. As the forward process acts independently on each latent element, we analyze a single element in the following. For analytical tractability, we measure their misalignment with an upper bound $\overline{\mathcal{W}}_t\ge\mathcal{W}(\hat{q}_t, q_t)$ derived below, and define a simple surrogate objective with a closed-form maximizer to trade off distributional alignment and content preservation:
\begin{equation}
	\mathcal{H}(\bar{\alpha}_t) = \sqrt{\bar{\alpha}_t} \cdot \exp\left(-\lambda \overline{\mathcal{W}}_t\right),
	\label{eq:h_func}
\end{equation}
where the signal coefficient $\sqrt{\bar{\alpha}_t}$ in Eq. \eqref{eq:gt_forward} measures content preservation and $\lambda > 0$ balances them.
\begin{proposition}
	\label{prop}
	For $\bar{\alpha}_t \in (0,1]$, the function $\mathcal{H}(\bar{\alpha}_t)$ admits a unique global maximum.
\end{proposition}
\begin{proof}
	Coupling $\hat{q}_t$ and $q_t$ through the same forward noise $\epsilon$ yields samples $\mathbf{z}'_t\sim\hat{q}_t$ and $\mathbf{z}_t\sim q_t$ with $\mathbf{z}'_t-\mathbf{z}_t=\sqrt{\bar{\alpha}_t}(\mathbf{z}_0-\mathbf{z}_d)$, where $\mathbf{z}_0\sim\mathcal{N}(\hat{\mathbf{z}}_0,\hat{\sigma}_0^2)$. Since the 2-Wasserstein distance is the infimum over all couplings, it is upper bounded by
	\begin{equation}
		\mathcal{W}(\hat{q}_t, q_t) \le \sqrt{\mathbb{E}(\mathbf{z}'_t-\mathbf{z}_t)^2} = \sqrt{\bar{\alpha}_t}\,\omega \triangleq \overline{\mathcal{W}}_t,
	\end{equation}
	where $\omega=\sqrt{(\hat{\mathbf{z}}_0 - \mathbf{z}_d)^2 + \hat{\sigma}_0^2}$. Substituting $\overline{\mathcal{W}}_t$ into Eq. \eqref{eq:h_func} gives $\mathcal{H}(\bar{\alpha}_t) = \sqrt{\bar{\alpha}_t}\, e^{-\lambda \omega \sqrt{\bar{\alpha}_t}}$, whose derivative is
	\begin{equation}
		\frac{d\mathcal{H}}{d\bar{\alpha}_t} = \frac{e^{-\lambda \omega \sqrt{\bar{\alpha}_t}}}{2\sqrt{\bar{\alpha}_t}} \bigl( 1 - \lambda \omega \sqrt{\bar{\alpha}_t} \bigr)
	\end{equation}
	vanishes at the critical point $\bar{\alpha}_t^* = 1/(\lambda \omega)^2$. For $\bar{\alpha}_t < \bar{\alpha}_t^*$, the derivative is positive; for $\bar{\alpha}_t > \bar{\alpha}_t^*$, it is negative. Therefore, $\mathcal{H}(\bar{\alpha}_t)$ has a unique global maximum on $(0,1]$ at $\min\{1,\bar{\alpha}_t^*\}$.
\end{proof}
This proposition shows that the trade-off between prompt fidelity and distributional alignment admits a unique optimal timestep, as illustrated in Fig. \ref{fig:t_demo}. During inference, since $\mathbf{z}_d$ is unavailable, we estimate $\omega$ from the predicted uncertainty. For well-calibrated uncertainty, $\mathbb{E}(\hat{\mathbf{z}}_0 - \mathbf{z}_d)^2\approx\hat{\sigma}_0^2$, so that $\omega\propto\hat{\sigma}_0$. As a single timestep is shared by the entire latent, we average the predicted standard deviations over all $N$ latent elements:
\begin{equation}
	\bar{\sigma}_0 = \frac{1}{N} \sum_{i=1}^N \hat{\sigma}_0^{(i)}.
\end{equation}
Absorbing the constant factor into $\tau\propto 1/\lambda$, the optimum becomes $(\tau/\bar{\sigma}_0)^2$. In practice, however, this rule drives $\hat{t}$ toward $T$ and erases the prompt content too quickly. We therefore adopt its square-root relaxation, which decays more gently:
\begin{equation}
	\bar{\alpha}_{\hat{t}} = \frac{\tau}{\bar{\sigma}_0},
	\label{eq:alpha_hat}
\end{equation}
which is clamped to $[\bar{\alpha}_{min}, 1]$, consistent with the boundary case of Proposition \ref{prop}. Given $\hat{t}$, a lightweight noise sampling network $\mathcal{S}(\cdot)$ takes as input the RGB latent $\mathbf{z}_x$, the mean $\hat{\mathbf{z}}_t=\sqrt{\bar{\alpha}_{\hat{t}}}\,\hat{\mathbf{z}}_0$ and standard deviation $\hat{\sigma}_t=\sqrt{\bar{\alpha}_{\hat{t}}\,\hat{\sigma}^2_0+1-\bar{\alpha}_{\hat{t}}}$ of $\hat{q}_{\hat{t}}$, and outputs a noise prediction $\hat{\epsilon}=\mathcal{S}(\mathbf{z}_x,\hat{\mathbf{z}}_t,\hat{\sigma}_t)$. The intermediate noisy latent is formulated as
\begin{equation}
	\hat{\epsilon}_t=\hat{\mathbf{z}}_t+\hat{\sigma}_t\,\hat{\epsilon}.
\end{equation}
Note that although $\hat{t}$ is determined by the spatially averaged uncertainty $\bar{\sigma}_0$, the pixel-wise uncertainty $\hat{\sigma}_0$ is retained in $\hat{\sigma}_t$, so that less reliable regions receive stronger noise perturbation and are thus more strongly corrected by the diffusion prior.
Finally, $\hat{\epsilon}_t$ is fed into the pre-trained Marigold-LCM model for one-step denoising at timestep $\hat{t}$ following Eq. \eqref{eq:lcm_inference}, yielding the final depth latent $\hat{\mathbf{z}}_d=f_\theta(\mathbf{z}_x\Vert\hat{\epsilon}_t,\hat{t})$, which is decoded to the final high-resolution depth map $\hat{\mathbf{d}} = \mathcal{D}_l(\hat{\mathbf{z}}_d)$. To train $\mathcal{S}(\cdot)$ with Marigold-LCM kept frozen, we employ a reconstruction loss that encourages accurate recovery of both the latent and pixel-space depth:
\begin{equation}
	\mathcal{L}_{\text{rec}} = \|\hat{\mathbf{z}}_d - \mathbf{z}_d\|_2^2 + \|\hat{\mathbf{d}} - \mathbf{d}\|_1.
\end{equation}
Additionally, we incorporate a gradient loss to encourage spatial smoothness and edge preservation in the predicted depth map:
\begin{equation}
	\mathcal{L}_g=\Vert\partial_x\hat{\mathbf{d}}-\partial_x\mathbf{d}\Vert_1+\Vert\partial_y\hat{\mathbf{d}}-\partial_y\mathbf{d}\Vert_1,
\end{equation}
where $\partial_x$ and $\partial_y$ denote the horizontal and vertical spatial gradients, respectively. 

\begin{table}[t]
	\renewcommand\arraystretch{1.1}
	\setlength{\tabcolsep}{2.8pt}
	\caption{
		Quantitative zero-shot comparison on real-world benchmarks. Input resolutions for the low-resolution depth prompt are provided in the second row. For each metric, the best and second-best results are highlighted in \textbf{bold} and \uline{underline}, respectively.
	}
	\label{tab:real_world}
	\centering
	\resizebox{\linewidth}{!}{
		\begin{tabular}{@{}lccccccccccccc@{}}
			\toprule
			\multirow{3}{*}{Method} & \multirow{3}{*}{Ref.} & \multicolumn{6}{c}{RGB-D-D}                                                                                                      & \multicolumn{6}{c}{TOFDSR}                                                                                                              \\ \cmidrule(l){3-8} \cmidrule(l){9-14} 
			&                       & \multicolumn{2}{c}{$192\times 144$}         & \multicolumn{2}{c}{$64\times 48$}    & \multicolumn{2}{c}{$32\times 24$}           & \multicolumn{2}{c}{$192\times 144$}         & \multicolumn{2}{c}{$64\times 48$}           & \multicolumn{2}{c}{$32\times 24$}           \\ \cmidrule(l){3-4} \cmidrule(l){5-6} \cmidrule(l){7-8} \cmidrule(l){9-10} \cmidrule(l){11-12} \cmidrule(l){13-14}    
			&                       & RMSE $\downarrow$    & MAE $\downarrow$     & RMSE $\downarrow$ & MAE $\downarrow$ & RMSE $\downarrow$    & MAE $\downarrow$     & RMSE $\downarrow$    & MAE $\downarrow$     & RMSE $\downarrow$    & MAE $\downarrow$     & RMSE $\downarrow$    & MAE $\downarrow$     \\ \midrule
			DA v2-L             & NeurIPS-24          & 0.1082               & 0.0745               & 0.1082               & 0.0745           & 0.1082               & 0.0745               & 0.1035               & 0.0690               & 0.1035               & 0.0690               & 0.1035               & 0.0690               \\
			MG-LCM                  & CVPR-24             & 0.1306               & 0.1007               & 0.1306            & 0.1007           & 0.1306               & 0.1007               & 0.1219               & 0.0835               & 0.1219               & 0.0835               & 0.1219               & 0.0835               \\ 
			DA v3-NGL               & ICLR-26               & 0.2835            & 0.2528           & 0.2835            & 0.2528           & 0.2835            & 0.2528           & 0.5229            & 0.3806           & 0.5229            & 0.3806           & 0.5229            & 0.3806           \\ \midrule
			SGNet                   & AAAI-24             & 0.0726               & 0.0388               & 0.0723            & 0.0391           & 0.0740               & 0.0416               & 0.0726               & 0.0400               & 0.0726               & 0.0404               & 0.0767               & 0.0433               \\
			C2PD                    & AAAI-25             & 0.0724               & 0.0388               & 0.0717            & 0.0384           & 0.0714               & {\ul 0.0384} & {\ul 0.0723}         & 0.0403               & {\ul 0.0718}         & 0.0400 & {\ul 0.0722}         & {\ul 0.0405}         \\
			DuCos                   & ICCV-25             & 0.0744               & 0.0388               & 0.0741            & 0.0385           & 0.0898               & 0.0555               & 0.0742               & {\ul 0.0399}         & 0.0742               & {\ul 0.0399} & 0.0958               & 0.0594               \\
			PromptDA                & CVPR-25             & 0.0706               & 0.0506               & 0.0705            & 0.0505           & 0.0704               & 0.0503               & 0.0879               & 0.0623               & 0.0885               & 0.0624               & 0.0884               & 0.0620               \\
			OMNI-DC                 & ICCV-25             & 0.0680               & {\ul 0.0371}         & 0.0687            & 0.0385           & 0.0846               & 0.0468               & 0.0748               & {\ul 0.0399}         & 0.0890               & 0.0451               & 0.1071               & 0.0535               \\
			MG-DC                   & ICCV-25             & 0.0685            & 0.0382           & 0.0684            & 0.0392           & 0.0826            & 0.0475           & 0.0780            & 0.0416           & 0.0844            & 0.0441           & 0.0980            & 0.0508           \\
			PriorDA                 & ICLR-26            & {\ul 0.0655}         & 0.0379               & {\ul 0.0613}      & {\ul 0.0377}     & {\ul 0.0691}         & 0.0432 & 0.0731               & 0.0411               & 0.0755               & 0.0435               & 0.0833               & 0.0480               \\ \midrule
			\textbf{Ours}           & ---                   & \textbf{0.0607}      & \textbf{0.0332}      & \textbf{0.0597}   & \textbf{0.0330}  & \textbf{0.0597}      & \textbf{0.0342}      & \textbf{0.0693}      & \textbf{0.0336}      & \textbf{0.0688}      & \textbf{0.0346}      & \textbf{0.0700}      & \textbf{0.0363}      \\ \bottomrule
		\end{tabular}
	}
	\vskip -0.05in
\end{table}

\subsection{Implementation Details}
\label{sec:detail}
\paragraph{Dataset Details.}
We train AdaDS on the Hypersim dataset \cite{hypersim}, which provides high-quality synthetic RGB-depth pairs rendered from photorealistic indoor scenes. After filtering out samples with large invalid depth regions, we obtain around 56K training samples. LR depth prompts are generated by downsampling the ground-truth depth with random scaling factors, followed by augmentations that mimic real-world sensor degradations, including Gaussian noise, random filtering, and sparse point removal, where the removed pixels are refilled by KNN-based inpainting.

%We train our framework using Hypersim \cite{hypersim} dataset, due to its high-quality synthetic depth maps. We filter out samples with large invalid depth regions, and the final training set contains around 56K samples. We downsample ground truth depth maps by random scales to construct the LR input. Meanwhile, to simulate the real-world degradation patterns, we perform various degradation operations to the downsampled depth maps, including random noise, random image filtering and random point replacement. For the last operation, we randomly remove pixels in the LR input, and complete them using a KNN-based approache, simulating the sparse measurement of real-world range sensors.

\paragraph{Network Details.}
The calibration stage employs the ViT-Small variant of Depth Anything V2 as its image backbone, with the patch size changed to 16 to align its feature resolution with the latent space. Before encoding, depth maps are normalized to $[-1,1]$ using the 2nd and 98th percentiles of valid depth values in the prompt, and the decoded prediction is mapped back to metric scale by the inverse transformation. Both the depth branch and the noise sampling module adopt lightweight UNet-like architectures built from ConvNeXt \cite{convnext} blocks and transposed convolutions. Overall, AdaDS contains approximately 40.3M trainable parameters.
%
%We use the ViT-small version of Depth Anything V2 as the backbone of the coarse refinement stage. For feature extractor and noise sampling module, the ConvNeXt \cite{convnext} blocks and transposed convolutions are used as basis components for network construction. The total trainable parameters of AdaDS are around 45M.

\paragraph{Training Details.}
We train the two stages of AdaDS independently, with the VAE and Marigold-LCM kept frozen. The calibration stage is trained for 42K iterations, and the sampling stage for 35K iterations, both with a batch size of 8. Training is performed on 4 NVIDIA RTX 4090 GPUs and takes approximately 14 hours in total. Following the OneCycle policy \cite{onecycle}, the learning rate warms up from $10^{-5}$ to $10^{-4}$ over the first 10\% of iterations and then decays to $10^{-5}$ with a cosine schedule. The two stages are trained with $\mathcal{L}_1 = \mathcal{L}_{\text{nll}} + \gamma \mathcal{L}_d$ and $\mathcal{L}_2 = \mathcal{L}_{\text{rec}} + \mathcal{L}_g$, respectively, where the balancing hyperparameter is set to $\gamma = 0.5$. For timestep selection, we set $\tau=0.14$ and $\bar{\alpha}_{min}=0.6$.

\section{Experiment}

\subsection{Experiment Setup}

We evaluate the zero-shot generalization of AdaDS on both real-world and synthetic benchmarks. For \textit{real-world} evaluation, we use the official test splits of \textbf{RGB-D-D} \cite{rgbdd} and \textbf{TOFDSR} \cite{tofdsr}, containing 405 and 560 samples, respectively. Both datasets pair LR depth maps captured by a mobile ToF camera at $192 \times 144$ with ground-truth depth acquired by a high-end ToF sensor at $512 \times 384$, and varying upsampling factors are simulated by further bicubic downsampling of the LR inputs. For \textit{synthetic} evaluation, we use \textbf{NYUv2} \cite{nyu} (654 test samples) and \textbf{ScanNet} \cite{scannet} (500 uniformly sampled test frames) at $640\times480$, where LR prompts are synthesized by bicubic downsampling followed by common sensor artifacts, including Gaussian noise, Gaussian blur, sparse measurement, and low-bit compression. Invalid ground-truth pixels are filled by colorization-based inpainting \cite{colorization} before degradation, while all metrics are computed only on valid pixels. We report RMSE and MAE for absolute accuracy, and the threshold accuracy $\delta_{1.05}$, i.e., the percentage of pixels satisfying $\max(\hat{\mathbf{d}}/\mathbf{d},\mathbf{d}/\hat{\mathbf{d}})<1.05$, for relative precision. For baselines, predictions of affine-invariant methods (DA v2-L and MG-LCM) are aligned to the ground truth via Theil--Sen regression, DA v3-NGL adopts its metric DA3NESTED-GIANT-LARGE-1.1 variant, and depth completion methods (OMNI-DC and MG-DC) take the LR depth values placed at their corresponding positions on the high-resolution grid as sparse inputs. Image-only methods thus yield identical results across prompt settings.

\begin{table}
	\renewcommand\arraystretch{1.1}
	\setlength{\tabcolsep}{2.5pt}
	\caption{Quantitative zero-shot comparison on synthetic benchmarks. DS., GN., GB., SM. and LC. denote downsampling, Gaussian noise, Gaussian blur, sparse measurement and low-bit compression, respectively. RMSE metric is reported in this table.}
	\label{tab:synthetic}
	\centering
	\resizebox{\linewidth}{!}{
		\begin{tabular}{@{}lccccccccccc@{}}
			\toprule
			\multirow{2}{*}{Method} & \multirow{2}{*}{Ref.} & \multicolumn{5}{c}{ScanNet}                                                             & \multicolumn{5}{c}{NYUv2}                                                               \\ 
			\cmidrule(l){3-7} \cmidrule(l){8-12} 
			&                       & $16\times$ DS.           & DS. + GN.     & DS. + GB.     & DS. + SM.     & DS. + LC.     & $16\times$ DS.           & DS. + GN.     & DS. + GB.     & DS. + SM.     & DS. + LC.     \\ \midrule
			DA v2-L                   & NeurIPS-24            & 0.1416          & 0.1416          & 0.1416          & 0.1416          & 0.1416          & 0.2075         & 0.2075          & 0.2075          & 0.2075          & 0.2075          \\
			MG-LCM                  & CVPR-24               & 0.1986          & 0.1986          & 0.1986          & 0.1986          & 0.1986          & 0.2587          & 0.2587          & 0.2587          & 0.2587          & 0.2587          \\ 
			DA v3-NGL               & ICLR-26               & 0.2186          & 0.2186          & 0.2186          & 0.2186          & 0.2186          & 0.3095          & 0.3095          & 0.3095          & 0.3095          & 0.3095          \\ \midrule
			PromptDA                & CVPR-25               & {\ul 0.0927}    & 0.0983          & {\ul 0.0935}    & {\ul 0.0953}    & 0.0950          & {\ul 0.1326}    & {\ul 0.1352}    & {\ul 0.1349}    & {\ul 0.1369}    & {\ul 0.1338}    \\
			OMNI-DC                 & ICCV-25               & 0.1173          & 0.1224          & 0.1127          & 0.1188          & 0.1206          & 0.1966          & 0.2011          & 0.1894          & 0.2003          & 0.1980          \\
			MG-DC             & ICCV-25               & 0.1222          & 0.1288          & 0.1173          & 0.1214          & 0.1262          & 0.2027          & 0.2077          & 0.1951          & 0.2025          & 0.2052          \\
			PriorDA                 & ICLR-26              & 0.0933          & {\ul 0.0949}    & 0.0953          & 0.0954          & {\ul 0.0944}    & 0.1432          & 0.1444          & 0.1474          & 0.1458          & 0.1439          \\ \midrule
			\textbf{Ours}           & ---                   & \textbf{0.0754} & \textbf{0.0783} & \textbf{0.0810} & \textbf{0.0834} & \textbf{0.0773} & \textbf{0.1197} & \textbf{0.1217} & \textbf{0.1289} & \textbf{0.1315} & \textbf{0.1206} \\ \bottomrule
		\end{tabular}
	}
	 \vskip -0.05in
\end{table}

\subsection{Comparison with the State-of-the-Art}
We compare AdaDS against a range of state-of-the-art (SOTA) methods, including monocular depth estimation methods, i.e., Depth Anything V2 (DA v2-L) \cite{dav2}, Marigold-LCM (MG-LCM) \cite{marigold} and Depth Anything V3 (DA v3-NGL) \cite{dav3}, DSR methods, i.e., SGNet \cite{sgnet}, C2PD \cite{c2pd} and DuCos \cite{yan2025ducos}, as well as prompted depth estimation methods, i.e., PromptDA \cite{promptda}, OMNI-DC \cite{omni_dc}, Marigold-DC (MG-DC) \cite{marigold_dc} and PriorDA \cite{priorda}.

\paragraph{Real-World Evaluation.}
\begin{wrapfigure}{r}{0.51\linewidth}
	\vspace{-0.1in}
	\centering
	\includegraphics[width=\linewidth]{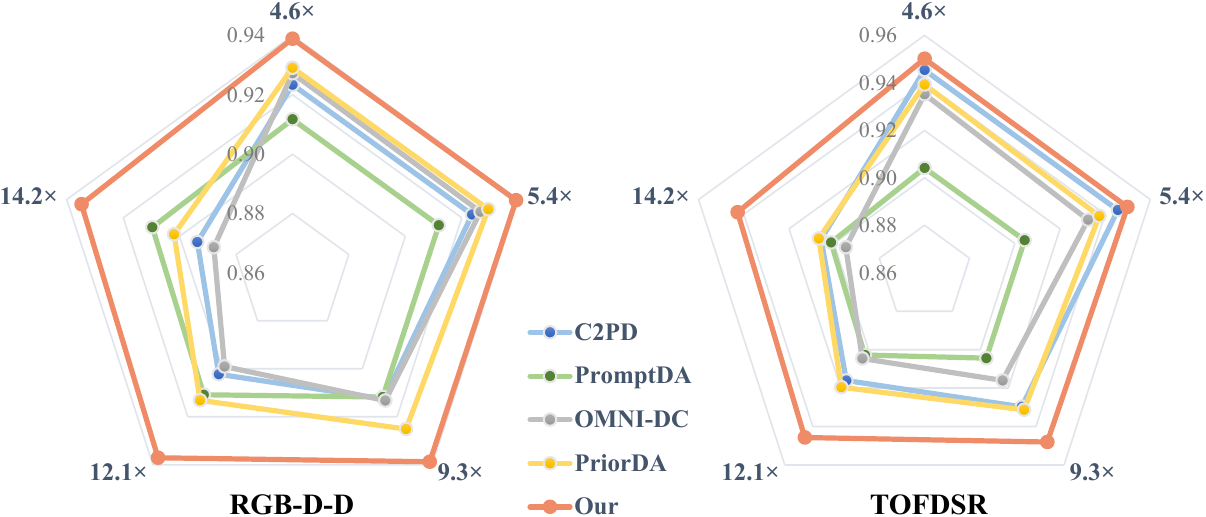}
	\caption{
		Comparison on real-world benchmarks with arbitrary upsampling factors. The $\delta_{1.05}$ $\uparrow$ metric is reported in this figure.
	}
	\label{fig:arbitrary}
	\vspace{-0.1in}
\end{wrapfigure}
Quantitative results on the real-world benchmarks are reported in Tab. \ref{tab:real_world}. AdaDS consistently outperforms all competing methods across every metric and upsampling factor. For instance, at the original input resolution, AdaDS outperforms the second-best methods, PriorDA on RGB-D-D and C2PD on TOFDSR, by 7.33\% and 4.15\% in RMSE, respectively. Notably, AdaDS exhibits remarkable stability with respect to input resolution changes, maintaining nearly consistent performance across all tested upsampling factors. In contrast, competing methods show clear brittleness. For example, the RMSE of MG-DC deteriorates by 20.58\% when the upsampling factor increases from 2.7$\times$ to 16$\times$. We further evaluate generalization to arbitrary (non-integer) upsampling factors in Fig. \ref{fig:arbitrary}, where AdaDS achieves the highest $\delta_{1.05}$ under all factors.

\begin{figure}[t]
	\centering
	\includegraphics[width=\linewidth]{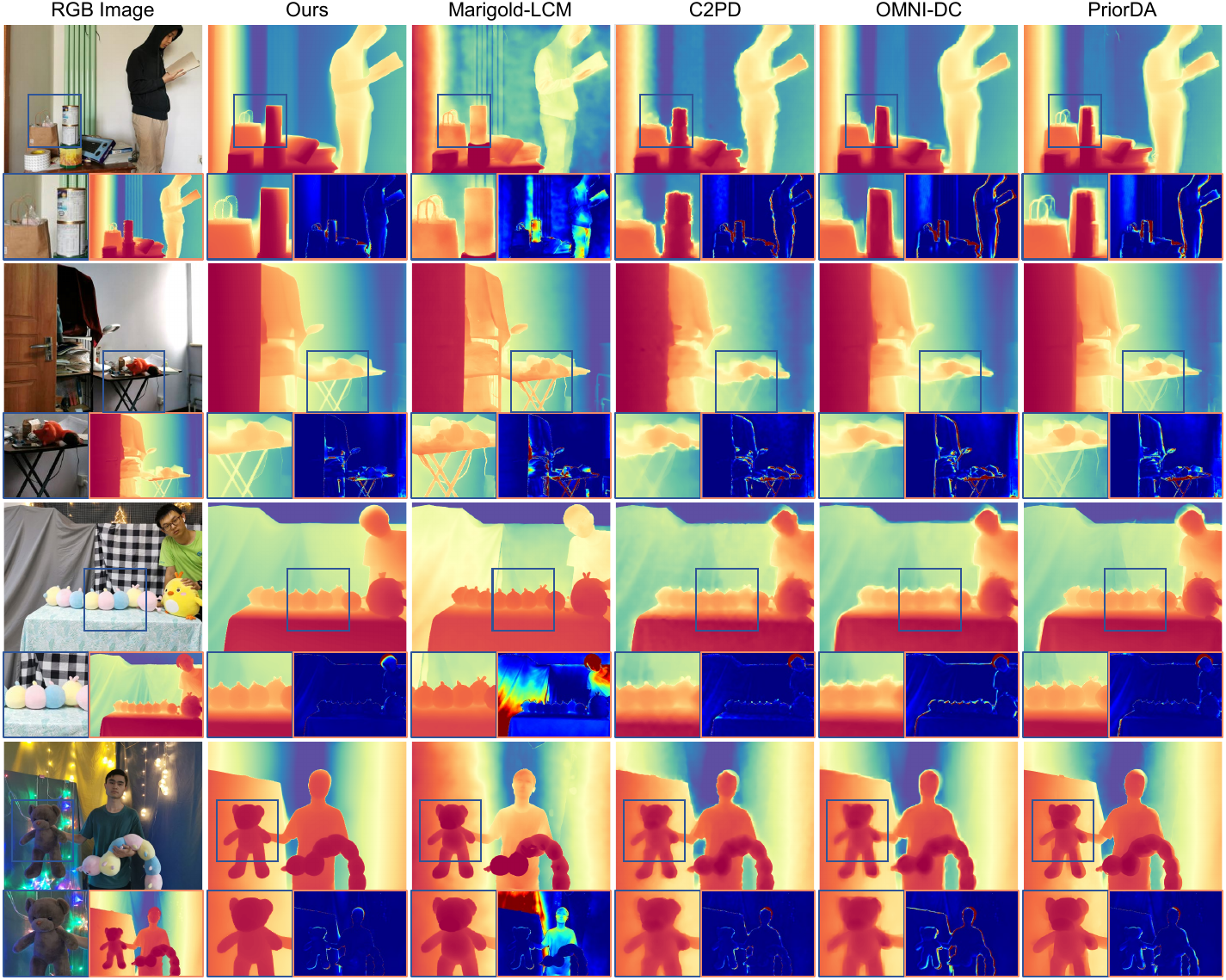}
	\caption{
		Qualitative comparison on real-world benchmarks. The first two samples are from RGB-D-D and the last two from TOFDSR. Blue boxes mark regions enlarged in the bottom-left inset of each result, and error maps are shown in the bottom-right inset, with warmer colors denoting larger errors.
%		Qualitative comparison on real-world datasets. The first two samples are from RGB-D-D dataset and the last two samples are from TOFDSR dataset. Error maps are shown in the bottom right side of each method, with red indicating larger prediction error.
	}
	\label{fig:real_world}
	\vskip -0.05in
\end{figure}

\begin{figure}[h]
	\centering
	\includegraphics[width=\linewidth]{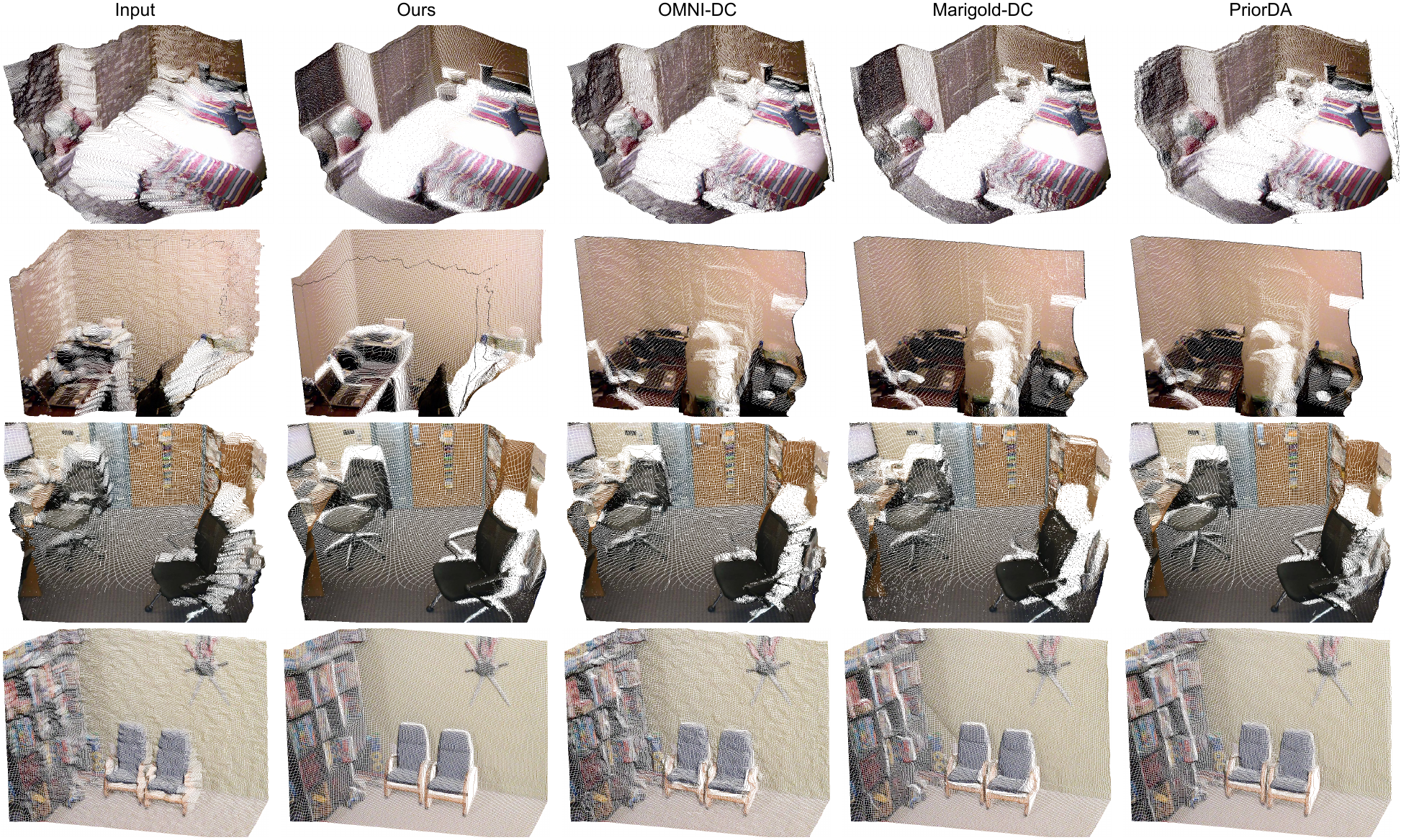}
	\caption{
		Qualitative comparison on synthetic benchmarks. The first two samples are from NYUv2 and the last two from ScanNet. To facilitate 3D structural evaluation, the input and estimated depth maps are back-projected into point clouds and colorized with the corresponding RGB images.
%		Qualitative comparison on synthetic datasets. The first and second samples are from NYUv2 and Scannet, respectively. We back-project estimated depth map to pointcloud and use corresponding rgb image for colorization.
		}
	\label{fig:synthetic}
	% \vskip -0.1in
\end{figure}

\paragraph{Synthetic Evaluation.}
Quantitative results on synthetic benchmarks are summarized in Tab. \ref{tab:synthetic}, where the ground-truth depth is downsampled by $16\times$ and further corrupted by different degradations. AdaDS significantly outperforms all baselines across all degradation types and both datasets. On average, it improves over the second-best methods by 16.46\% (PriorDA) on ScanNet and 7.57\% (PromptDA) on NYUv2, highlighting its strong robustness to compound depth degradations. Compared with the strongest image-only method, DA v2-L, AdaDS reduces RMSE by more than 36\% in all settings, showing that metric cues remain valuable even when the depth prompts are degraded.

%The quantitative results on synthetic benchmarks are reported in Tab. \ref{tab:synthetic}. To access the robustness of our method on sever degradations, we first downsample the ground truth by 16 times, and add various degradation patters to simulate the common degradation cases in real-world. Overall, our method significantly outperforms all competitors on all degradation patterns across two benchmarks, demonstrating superior robustness against various degradation patterns. Specifically, our method achieves an average improvement of 16.91\% and 7.57\% over the second best method, PromptDA, on ScanNet and NYUv2, respectively.

\paragraph{Qualitative Comparison.}
Qualitative results on real-world benchmarks are presented in Fig. \ref{fig:real_world}. AdaDS produces depth maps with substantially lower overall prediction error (as visualized in the error maps) while recovering fine-grained geometric details and sharp boundaries in the highlighted regions, outperforming all competitors. On synthetic benchmarks shown in Fig. \ref{fig:synthetic}, we back-project the predicted depth maps into 3D point clouds to better illustrate structural fidelity. Our method effectively restores severely distorted structures from heavily degraded inputs, preserving both local details and global scene geometry with noticeably higher quality than existing approaches.
%
%The qualitative comparison on real-world benchmarks are shown in Fig. \ref{fig:real_world}. Our method achieves lower overall prediction error as shown in the error map while preserving fine-grained details as illustrated in highlighted regions. The qualitative results on synthetic benchmarks are reported in Fig. \ref{fig:synthetic}. We back-project the estimated depth map to pointcloud to better compare their structural quality. Overall, our method can correctly recover the degraded structure from the input that suffers from strong structural distortion.

\section{Ablation Study}

We perform comprehensive ablation studies on real-world benchmarks, with quantitative results reported in Tab. \ref{tab:ablation}. 
Our calibration stage (b), built upon the DA v2-Small backbone (a), effectively fuses the LR depth prompt with RGB guidance and substantially improves over (a), reducing RMSE by 50.43\% on RGB-D-D and 44.57\% on TOFDSR. This confirms the importance of incorporating metric depth cues. To examine the role of diffusion sampling, variant (c) replaces the noise sampling and diffusion operations with direct regression of the final depth latent $\mathbf{z}_d$ , leading to RMSE increases of 2.68\% on RGB-D-D and 3.20\% on TOFDSR compared with full AdaDS. Variant (d) further projects the predicted latent from (c) into the forward diffusion process and applies MG-LCM denoising without learned noise adaptation, yielding inferior results and showing that naive forward projection is insufficient. Replacing our adaptive timestep selection rule in Eq. \eqref{eq:alpha_hat} with uniformly random timesteps (e) causes a clear performance drop, validating the need to balance prompt preservation and distributional alignment. Similarly, substituting the learned noise prediction $\hat{\epsilon}$ with standard Gaussian noise (f) substantially degrades performance, demonstrating the effectiveness of the sampling module $\mathcal{S}(\cdot)$ in positioning the intermediate latent within the high-probability region of the target posterior.

\begin{figure}
	\centering
	
	\begin{minipage}{0.485\textwidth}
		\centering
		\renewcommand\arraystretch{1.1}
		\setlength{\tabcolsep}{3.0pt}
		\resizebox{\linewidth}{!}{
			\begin{tabular}{@{}lcccc@{}}
				\toprule
				\multirow{2}{*}{}                                & \multicolumn{2}{c}{RGB-D-D}            & \multicolumn{2}{c}{TOFDSR}           \\ 
				\cmidrule(l){2-3}\cmidrule(l){4-5} 
				& RMSE $\downarrow$ & MAE $\downarrow$ & RMSE $\downarrow$ & MAE $\downarrow$ \\ \midrule
				(a) DA v2-S Baseline                             & 0.1265            & 0.0900           & 0.1252            & 0.0740           \\
				(b) Our Calibration Stage                            & 0.0627            & 0.0358           & {\ul 0.0694}      & 0.0379           \\ \midrule
				(c) (b) + Directly Regressing $\mathbf{z}_d$          & {\ul 0.0613}      & {\ul 0.0346}     & 0.0710            & {\ul 0.0379}     \\
				(d) (c) + MG-LCM Denoising                       & 0.0622            & 0.0367           & 0.0723            & 0.0391           \\ \midrule
				(e) Replace $\hat{t}$ with Random $t$       & 0.0768            & 0.0497           & 0.0809            & 0.0455           \\
				(f) Replace $\hat{\epsilon}$ with $\mathcal{N}(0, \mathbf{I})$ & 0.0668            & 0.0387           & 0.0742            & 0.0394           \\ \midrule
				\textbf{Our}                                     & \textbf{0.0597}   & \textbf{0.0330}  & \textbf{0.0688}   & \textbf{0.0346}  \\ \bottomrule
			\end{tabular}
		}
		\captionof{table}{Ablation study at $8\times$ factor. $\mathbf{z}_d$ , $\hat{t}$ and $\hat{\epsilon}$ represent GT depth latent, predicted denoising timestep and sampling noise, respectively.}
		\label{tab:ablation}
	\end{minipage}
	\hfill
	\begin{minipage}{0.48\textwidth}
		\vskip -0.02in
		\centering
		\includegraphics[width=\linewidth]{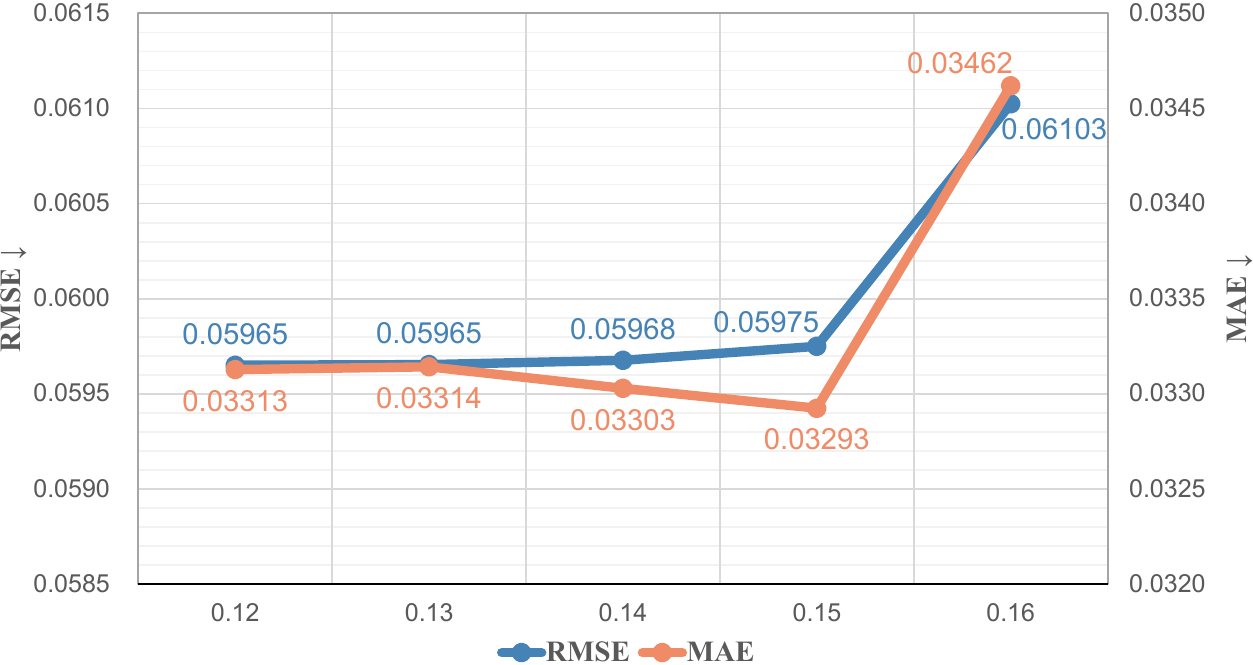}
		\vskip -0.02in
		\captionof{figure}{Validation experiment on timestep selection parameter $\tau$. The x axis indicates $\tau$ values. All models are re-trained using indicated $\tau$.}  % All models are separately trained using the indicated $\tau$.
		\label{fig:w_thrd}
	\end{minipage}
	
	% \vskip -0.1in
\end{figure}

%\begin{figure}
%	\centering
%	\includegraphics[width=0.98\linewidth]{fig/two_stage.pdf}
%	\caption{
%		Visual comparison between the intermediate result from the calibration stage and the final depth prediction after two stages.}
%	\label{fig:two_stage}
%	\vskip -0.15in
%\end{figure}

We further investigate the sensitivity to the threshold $\tau$ in the timestep selection rule. Results are shown in Fig. \ref{fig:w_thrd}. AdaDS remains robust for small values of $\tau$, but performance deteriorates significantly at $\tau = 0.16$. Larger $\tau$ selects very small timesteps that preserve input content but severely limit the model's ability to correct upstream errors. We adopt $\tau = 0.14$ in all experiments, as it provides the best trade-off between RMSE and MAE.

%As Marigold-LCM supports multi-step sampling, we evaluate AdaDS using 1, 2, and 4 LCM inference steps and report the quantitative results in Tab. \ref{tab:multi_step}. Our standard method uses one-step denoising. For multi-step inference, we use the estimated timestep $\hat{t}$ as the starting point and uniformly distribute the denoising timesteps within $[0, \hat{t}]$. At each step, we use the sampling network $\mathcal{S}(\cdot)$ to predict the sampling point. Using 2 inference steps slightly improves the MAE but degrades the RMSE. Extending to 4 steps significantly hurts overall performance. This occurs because, while multi-step LCM can enhance details, the prediction errors originating from the sampling network accumulate over multiple steps. This eventually causes the prediction to drift from the true metric depth values.

As MG-LCM supports multi-step sampling, we evaluate AdaDS with 1, 2, and 4 inference steps, as reported in Tab. \ref{tab:multi_step}. Our standard one-step setting achieves the best RMSE on both datasets. Although 2 steps slightly improve MAE on RGB-D-D and 4 steps achieve the best MAE on TOFDSR, both multi-step variants degrade RMSE. This suggests that additional denoising steps may enhance local detail but reduce metric fidelity. Since each step depends on the sampling network $\mathcal{S}(\cdot)$ to predict the intermediate sampling point, errors can accumulate over multiple steps and cause the prediction to drift from the true metric depth. Therefore, one-step denoising offers the best overall trade-off.

\begin{figure}
	\centering
	
	\begin{minipage}{0.42\textwidth}
		\centering
		\renewcommand\arraystretch{1.1}
		\setlength{\tabcolsep}{8.0pt}
		\resizebox{\linewidth}{!}{
			\begin{tabular}{@{}ccccc@{}}
				\toprule
				\multirow{2}{*}{\begin{tabular}[c]{@{}c@{}}Inference\\ Steps\end{tabular}} & \multicolumn{2}{c}{RGB-D-D}          & \multicolumn{2}{c}{TOFDSR}           \\ \cmidrule(l){2-3}\cmidrule(l){4-5}
				& RMSE $\downarrow$ & MAE $\downarrow$ & RMSE $\downarrow$ & MAE $\downarrow$ \\ \midrule
				4 Steps                                                                         & 0.0657            & 0.0364           & 0.0733            & \textbf{0.0331}  \\
				2 Steps                                                                          & {\ul 0.0605}      & \textbf{0.0329}  & {\ul 0.0698}      & {\ul 0.0338}     \\ \midrule
				\textbf{1 Step}                                                                 & \textbf{0.0597}   & {\ul 0.0330}     & \textbf{0.0688}   & 0.0346           \\ \bottomrule
			\end{tabular}
		}
		\captionof{table}{Ablation study of multi-step inference on real-world benchmarks.}
		\label{tab:multi_step}
	\end{minipage}
	\hfill
	\begin{minipage}{0.54\textwidth}
		\vskip -0.01in
		\centering
		\includegraphics[width=\linewidth]{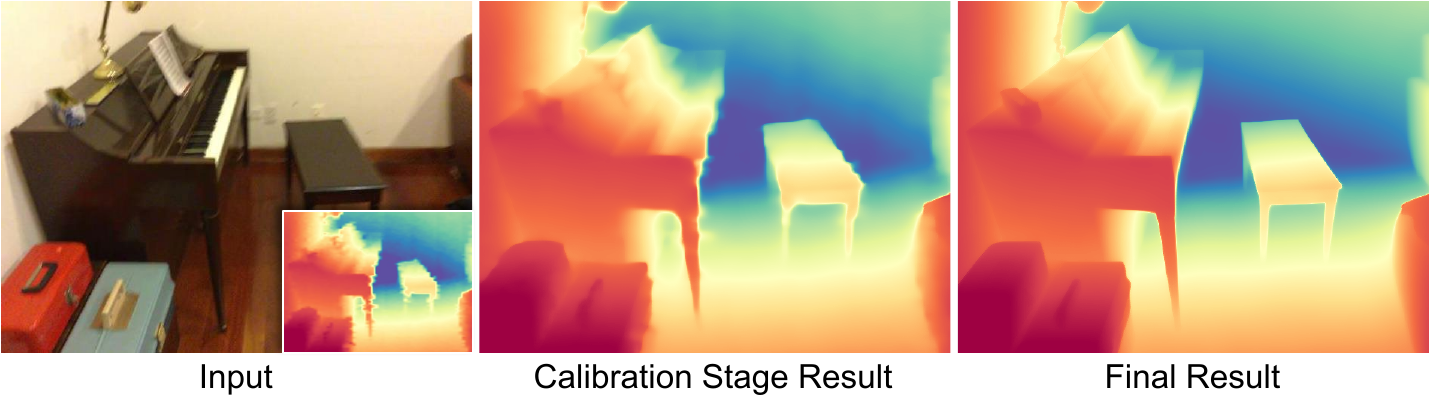}
		\vskip -0.02in
		\captionof{figure}{Visual comparison between the intermediate result from calibration stage and the final prediction.}  % All models are separately trained using the indicated $\tau$.
		\label{fig:two_stage}
	\end{minipage}
	
	% \vskip -0.05in
\end{figure}

Fig. \ref{fig:two_stage} further compares the intermediate output of the calibration stage with the final prediction. When the input depth is heavily degraded, the calibration stage recovers a large portion of the structural content and substantially narrows the distributional gap to high-quality depth, although residual inaccuracies remain. Based on this intermediate estimate, our adaptive noise sampling and diffusion denoising effectively correct these imperfections, producing a high-fidelity final depth map.

\begin{figure}
	\centering
	\begin{minipage}{0.48\textwidth}
		\centering
		\renewcommand\arraystretch{1.1}
		\setlength{\tabcolsep}{6.0pt}
		\resizebox{\linewidth}{!}{
			\begin{tabular}{@{}cccccc@{}}
				\toprule
				\multirow{2}{*}{$\eta$} & \multirow{2}{*}{$\kappa$} & \multicolumn{2}{c}{RGB-D-D} & \multicolumn{2}{c}{TOFDSR} \\ \cmidrule(l){3-4} \cmidrule(l){5-6}
				& & RMSE $\downarrow$ & MAE $\downarrow$ & RMSE $\downarrow$ & MAE $\downarrow$ \\ \midrule
				0.1 & 0.5   & 0.0599 & 0.0333 & 0.0689 & 0.0353 \\
				0.2 & 0.25  & 0.0602 & 0.0338 & 0.0688 & 0.0355 \\
				0.3 & 0.125 & 0.0609 & 0.0346 & 0.0692 & 0.0360 \\ \midrule
				\multicolumn{2}{c}{\textbf{Ours}} & 0.0597 & 0.0330 & 0.0688 & 0.0346 \\ \bottomrule
			\end{tabular}
		}
		\captionof{table}{Robustness to perturbed calibration-stage estimates at $8\times$ factor. Larger $\eta$ and smaller $\kappa$ denote stronger perturbations.}
		\label{tab:robust_main}
	\end{minipage}
	\hfill
	\begin{minipage}{0.485\textwidth}
		\centering
		\renewcommand\arraystretch{1.1}
		\setlength{\tabcolsep}{6.0pt}
		\resizebox{\linewidth}{!}{
			\begin{tabular}{@{}lccc@{}}
				\toprule
				Method & Time (s) & RGB-D-D & TOFDSR \\ \midrule
				MG-LCM \cite{marigold}         & 0.199 & 0.1306 & 0.1219 \\
				ControlNet \cite{controlnet}   & 0.226 & 0.0664 & 0.0716 \\
				PriorDA \cite{priorda}         & 0.108 & 0.0613 & 0.0755 \\ \midrule
				Our Calibration Stage          & 0.015 & 0.0627 & 0.0694 \\
				Our Sampling Network           & 0.003 & --     & --     \\
				\textbf{Ours Full}             & 0.217 & \textbf{0.0597} & \textbf{0.0688} \\ \bottomrule
			\end{tabular}
		}
		\captionof{table}{Efficiency comparison at $8\times$ factor, reporting RMSE and per-image inference time on a single RTX 4090 GPU.}
		\label{tab:efficiency_main}
	\end{minipage}
	\vskip -0.05in
\end{figure}

We further evaluate whether AdaDS can tolerate inaccurate calibration-stage estimates. Specifically, we perturb the refined latent as $\hat{\mathbf{z}}'_0=\hat{\mathbf{z}}_0+\hat{\mathbf{z}}_0\times\mathcal{N}(0,\eta)$ and suppress the predicted uncertainty as $\hat{\sigma}'^{2}_0=\hat{\sigma}^{2}_0\times\mathcal{U}(\kappa,1)$, where larger $\eta$ and smaller $\kappa$ introduce stronger latent perturbation and more severe uncertainty underestimation, respectively. As shown in Tab. \ref{tab:robust_main}, AdaDS remains stable under all settings, with only minor degradation even under the strongest perturbation. This indicates that the sampling stage corrects imperfect calibration outputs rather than accumulating first-stage errors.

Finally, we compare computational efficiency in Tab. \ref{tab:efficiency_main}, where all diffusion-based methods use one-step denoising and ControlNet \cite{controlnet} denotes a Marigold-based variant conditioned on the upsampled depth prompt. AdaDS achieves the best accuracy with a modest overhead, as its calibration stage and sampling network take only 0.015s and 0.003s, respectively. The main cost comes from the MG-LCM denoiser, suggesting that AdaDS can benefit from faster diffusion backbones.

\section{Limitations and Broader Impact}
\label{sec:limit_impact}
AdaDS is designed for low-resolution depth prompts and is trained and evaluated on indoor scenes, where such prompts are typically produced by ToF and structured-light sensors. It does not directly address monocular depth estimation or depth completion from sparse LiDAR or SfM measurements, whose uncertainty structures differ substantially. Moreover, as AdaDS relies on a pre-trained depth diffusion prior, its predictions may inherit biases or failure modes of this prior in rare scenes or domains far from its training distribution. Extending AdaDS to more general geometric guidance is a promising future direction. Regarding broader impact, AdaDS may reduce the reliance on expensive depth sensors and benefit robotics, AR/VR, and assistive perception. However, inaccurate estimates under out-of-distribution conditions could lead to unsafe decisions in safety-critical systems, and improved 3D perception may be misused for surveillance or privacy-invasive reconstruction, calling for careful validation and human oversight.

\section{Conclusion}
%Generalization and robustness remain persistent challenges in DSR tasks, particularly under diverse real-world degradations and unseen scenarios. We address this by exploiting the contraction property in diffusion models, where degraded inputs progressively align with high-quality depth distributions. Building on this insight, we propose AdaDS, a two-stage framework that performs uncertainty-aware calibration followed by adaptive noise injection to position intermediates in the high-probability region of the target posterior learned by a pre-trained diffusion model. Extensive experiments demonstrate AdaDS's superior zero-shot robustness to arbitrary upsampling factors and degradations.

Generalization remains a key challenge in prompted metric depth estimation under diverse real-world degradations. We address this by exploiting the contraction property of the forward diffusion process, under which degraded depth prompts progressively align with their high-quality counterparts while the prompt content is gradually erased. Based on this trade-off, AdaDS calibrates the prompt with uncertainty estimation, adaptively selects the starting timestep, and injects learned noise to place the intermediate sample in a high-density region of the depth distribution learned by a pre-trained diffusion model. Extensive experiments demonstrate its strong zero-shot robustness to arbitrary upsampling factors, diverse degradations, and real-world depth prompts.

%We address generalization challenges in prompted metric depth estimation under diverse real-world depth degradations. AdaDS exploits the contraction property of diffusion models, where degraded depth prompts progressively align with high-quality depth distributions, and performs uncertainty-aware prompt calibration with adaptive noise injection to effectively leverage a pre-trained diffusion prior. Experiments demonstrate strong zero-shot robustness to arbitrary upsampling factors, diverse degradations, and real-world depth prompts.

\begin{ack}
This work is supported by the Ministry of Education, Singapore, under its MOE Academic Research Fund Tier 2 (MOE-T2EP20124-0013), and partially supported by the Lee Kong Chian Fellowship awarded to ZHOU Pan by Singapore Management University.
\end{ack}

%\small
\bibliographystyle{ieeenat_fullname}
\bibliography{reference}

%%%%%%%%%%%%%%%%%%%%%%%%%%%%%%%%%%%%%%%%%%%%%%%%%%%%%%%%%%%%

% \clearpage
% \appendix

% \input{sec/7_supplement}

%%%%%%%%%%%%%%%%%%%%%%%%%%%%%%%%%%%%%%%%%%%%%%%%%%%%%%%%%%%%

\newpage
\section*{NeurIPS Paper Checklist}

\begin{enumerate}

\item {\bf Claims}
    \item[] Question: Do the main claims made in the abstract and introduction accurately reflect the paper's contributions and scope?
    \item[] Answer: \answerYes{} % Replace by \answerYes{}, \answerNo{}, or \answerNA{}.
    \item[] Justification: The main claims made in the abstract and introduction accurately reflect the main contributions and scope of this paper.

\item {\bf Limitations}
    \item[] Question: Does the paper discuss the limitations of the work performed by the authors?
    \item[] Answer: \answerYes{} % Replace by \answerYes{}, \answerNo{}, or \answerNA{}.
    \item[] Justification: We discuss the limitations and future work in Section \ref{sec:limit_impact}.  

\item {\bf Theory assumptions and proofs}
    \item[] Question: For each theoretical result, does the paper provide the full set of assumptions and a complete (and correct) proof?
    \item[] Answer: \answerYes{} % Replace by \answerYes{}, \answerNo{}, or \answerNA{}.
    \item[] Justification: We present a proposition in the Proposition \ref{prop} and provide a complete proof of it.

    \item {\bf Experimental result reproducibility}
    \item[] Question: Does the paper fully disclose all the information needed to reproduce the main experimental results of the paper to the extent that it affects the main claims and/or conclusions of the paper (regardless of whether the code and data are provided or not)?
    \item[] Answer: \answerYes{} % Replace by \answerYes{}, \answerNo{}, or \answerNA{}.
    \item[] Justification: We provide detailed implementation and experimental details in Sec. \ref{sec:detail}.

\item {\bf Open access to data and code}
    \item[] Question: Does the paper provide open access to the data and code, with sufficient instructions to faithfully reproduce the main experimental results, as described in supplemental material?
    \item[] Answer: \answerYes{} % Replace by \answerYes{}, \answerNo{}, or \answerNA{}.
    \item[] Justification: All datasets used in this paper are open access, and we plan to release our code at \url{https://github.com/w2kun/AdaDS}.

\item {\bf Experimental setting/details}
    \item[] Question: Does the paper specify all the training and test details (e.g., data splits, hyperparameters, how they were chosen, type of optimizer) necessary to understand the results?
    \item[] Answer: \answerYes{} % Replace by \answerYes{}, \answerNo{}, or \answerNA{}.
    \item[] Justification: We provide detailed implementation and experimental details in Sec. \ref{sec:detail}.

\item {\bf Experiment statistical significance}
    \item[] Question: Does the paper report error bars suitably and correctly defined or other appropriate information about the statistical significance of the experiments?
    \item[] Answer: \answerNo{} % Replace by \answerYes{}, \answerNo{}, or \answerNA{}.
    \item[] Justification: The paper reports deterministic benchmark results following standard evaluation protocols, but does not include error bars or statistical significance tests due to the high computational cost of repeated training and evaluation.

\item {\bf Experiments compute resources}
    \item[] Question: For each experiment, does the paper provide sufficient information on the computer resources (type of compute workers, memory, time of execution) needed to reproduce the experiments?
    \item[] Answer: \answerYes{} % Replace by \answerYes{}, \answerNo{}, or \answerNA{}.
    \item[] Justification: We provide sufficient information about the computer resources we used in our experiments in Sec. \ref{sec:detail}.
    
\item {\bf Code of ethics}
    \item[] Question: Does the research conducted in the paper conform, in every respect, with the NeurIPS Code of Ethics \url{https://neurips.cc/public/EthicsGuidelines}?
    \item[] Answer: \answerYes{} % Replace by \answerYes{}, \answerNo{}, or \answerNA{}.
    \item[] Justification: Our research conforms with the NeurIPS Code of Ethics.

\item {\bf Broader impacts}
    \item[] Question: Does the paper discuss both potential positive societal impacts and negative societal impacts of the work performed?
    \item[] Answer: \answerYes{} % Replace by \answerYes{}, \answerNo{}, or \answerNA{}.
    \item[] Justification: We discuss both potential positive societal impacts and negative societal impacts of our work in Sec. \ref{sec:limit_impact}.
    
\item {\bf Safeguards}
    \item[] Question: Does the paper describe safeguards that have been put in place for responsible release of data or models that have a high risk for misuse (e.g., pre-trained language models, image generators, or scraped datasets)?
    \item[] Answer: \answerNA{} % Replace by \answerYes{}, \answerNo{}, or \answerNA{}.
    \item[] Justification: Our work poses no such risks.

\item {\bf Licenses for existing assets}
    \item[] Question: Are the creators or original owners of assets (e.g., code, data, models), used in the paper, properly credited and are the license and terms of use explicitly mentioned and properly respected?
    \item[] Answer: \answerYes{} % Replace by \answerYes{}, \answerNo{}, or \answerNA{}.
    \item[] Justification: We properly cite the creators of all external datasets, pre-trained models, and baseline methods used in this work.

\item {\bf New assets}
    \item[] Question: Are new assets introduced in the paper well documented and is the documentation provided alongside the assets?
    \item[] Answer: \answerNA{} % Replace by \answerYes{}, \answerNo{}, or \answerNA{}.
    \item[] Justification: Our paper does not release new assets.

\item {\bf Crowdsourcing and research with human subjects}
    \item[] Question: For crowdsourcing experiments and research with human subjects, does the paper include the full text of instructions given to participants and screenshots, if applicable, as well as details about compensation (if any)? 
    \item[] Answer: \answerNA{} % Replace by \answerYes{}, \answerNo{}, or \answerNA{}.
    \item[] Justification: Our paper does not involve crowdsourcing nor research with human subjects.

\item {\bf Institutional review board (IRB) approvals or equivalent for research with human subjects}
    \item[] Question: Does the paper describe potential risks incurred by study participants, whether such risks were disclosed to the subjects, and whether Institutional Review Board (IRB) approvals (or an equivalent approval/review based on the requirements of your country or institution) were obtained?
    \item[] Answer: \answerNA{} % Replace by \answerYes{}, \answerNo{}, or \answerNA{}.
    \item[] Justification: Our paper does not involve crowdsourcing nor research with human subjects.

\item {\bf Declaration of LLM usage}
    \item[] Question: Does the paper describe the usage of LLMs if it is an important, original, or non-standard component of the core methods in this research? Note that if the LLM is used only for writing, editing, or formatting purposes and does \emph{not} impact the core methodology, scientific rigor, or originality of the research, declaration is not required.
    %this research? 
    \item[] Answer: \answerNA{} % Replace by \answerYes{}, \answerNo{}, or \answerNA{}.
    \item[] Justification: Our core method development in this research does not involve LLMs as any important, original, or non-standard components.

\end{enumerate}

\end{document}